\documentclass[times,twocolumn,final,numbers]{elsarticle}

\usepackage{prletters}
\usepackage{framed,multirow}

\usepackage{amsmath,amssymb,amsthm,bbm}
\usepackage[table]{xcolor}
\usepackage{color,soul}
\usepackage{graphicx}
\usepackage{booktabs}
\usepackage{todonotes}
\usepackage{float}
\usepackage{latexsym}
\usepackage{comment}
\usepackage{bm}
\usepackage{multirow}
\usepackage[shortlabels]{enumitem}
\usepackage{overpic}

\usepackage{algorithm}
\usepackage{algpseudocode}

\newtheorem{assumption}{Assumption}
\newtheorem{lemma}{Lemma}
\newtheorem{theorem}{Theorem}

\journal{Pattern Recognition Letters}

\makeatletter
\long\def\@makecaption#1#2{%
  \vskip\abovecaptionskip
  \sbox\@tempboxa{#1: #2}%
  \ifdim \wd\@tempboxa >\hsize
    #1: #2\par
  \else
    \global \@minipagefalse
    \hb@xt@\hsize{\hfil\box\@tempboxa\hfil}%
  \fi
  \vskip\belowcaptionskip}
\makeatother

\begin{document}

\begin{frontmatter}

\title{Error Estimate and Convergence Analysis for Data Valuation}

\author[label1]{Zhangyong Liang}\ead{zyliang1994@tju.edu.cn}
\author[label2]{Huanhuan Gao\corref{cor1}}\ead{gao\_huanhuan@jlu.edu.cn}
\author[label3]{Ji Zhang\corref{cor1}}\ead{ji.zhang@unisq.edu.au}
\cortext[cor1]{Corresponding author.}

\affiliation[label1]{organization={National Center for Applied Mathematics, Tianjin University},  
            city={Tianjin},
            postcode={300072},
            country={PR China}}
\affiliation[label2]{organization={School of Mechanical and Aerospace Engineering, Jilin University},
            city={Changchun},
            postcode={130025},
            country={PR China}}
\affiliation[label3]{organization={School of Mathematics, Physics and Computing, University of Southern Queensland},
            city={Toowoomba, Queensland},
            postcode={4350},
            country={Australia}}

\vspace{-16pt}

\begin{abstract}
Data valuation aims to quantify the contribution of individual data points to learning objectives, yet most existing approaches based on the Shapley value~\cite{shapley1953,ghorbani2019,wangdata} require multiple training runs and thus lack unified theoretical guarantees. The recently proposed neural dynamic data valuation (NDDV) framework~\cite{liang2024neural} overcomes this limitation by casting data valuation as a discrete-time stochastic optimal control problem, enabling analysis via the stochastic maximum principle. In this work, we establish rigorous error estimates and convergence results for NDDV. Under standard Lipschitz and smoothness assumptions, we derive quadratic error bounds for loss differences that scale inversely with the time discretization and quadratically with control perturbations. We further prove that the expected squared gradient norm of the training loss vanishes asymptotically and that the associated meta loss achieves a sublinear convergence rate. These results provide the first mathematical guarantees for the stability and convergence of dynamic data valuation methods.

\end{abstract}

\begin{keyword}
Data valuation, Neural dynamic data valuation, Error estimation, Convergence analysis. 
\vspace{.1cm}
\end{keyword}

\end{frontmatter}

\section{Introduction}

Data valuation aims to quantify the contribution of individual data points to a prescribed learning objective. A classical formulation adopts the Shapley value from cooperative game theory~\cite{shapley1953, elkind2016game}, which has been widely used in recent data valuation methods~\cite{pei2020,ghorbani2019,wangdata,jia2019knn,kwon2021beta,ilyas2022datamodels,li2023robust}.
For a dataset $\mathcal{D}$ of size $N$, the Shapley value assigns to each point $(x_i,y_i)$ the average marginal utility over all subsets of $\mathcal{D}$, as follows
\begin{align*}
\label{eqn:shapley-merged}
\phi_{\mathrm{Shap}}(x_i, y_i; U)
:= \frac{1}{N}\sum_{j=1}^{N}\frac{1}{\binom{N-1}{\,j-1\,}}
\sum_{S \in \mathcal{D}^{\backslash i}_{\,j-1}}
\Big[ U\big(S\cup\{(x_i,y_i)\}\big) - U(S) \Big].
\end{align*}
where $\mathcal{D}^{\backslash i}_{\,j-1} := \{\, S \subseteq \mathcal{D}\setminus\{(x_i,y_i)\} : |S|=j-1 \,\}$ collects all size-$(j\!-\!1)$ subsets $S$ excluding $(x_i,y_i)$, and the outer average is taken uniformly over coalition sizes $j=1,\dots,N$.

While this formulation is conceptually appealing, it requires evaluating marginal contributions across multiple training runs, making unified error estimation or convergence analysis intractable.

To address this limitation, Liang et al.~\cite{liang2024neural} introduced the neural dynamic data valuation (NDDV) framework, which recasts data valuation as a discrete-time stochastic optimal control problem. Within this formulation, the forward state dynamics evolve under a controlled drift, and the corresponding adjoint variables arise naturally through the stochastic maximum principle (SMP)~\cite{kushner1965,peng1990smp}. This structure allows all data values to be computed simultaneously within a single training process, enabling mathematical analysis via the associated Hamiltonian system.

The objective of this work is to develop rigorous analytical guarantees for NDDV. Under standard Lipschitz and smoothness assumptions, we derive a quadratic error estimate for the discrete-time controlled system and establish stability with respect to control perturbations. Furthermore, we prove that the expected squared gradient norm of the training loss vanishes asymptotically and that the meta-level optimization achieves a sublinear convergence rate. These results provide, to our knowledge, the first mathematical error and convergence guarantees for dynamic data valuation.

The main contributions of this work are summarized as follows:
\begin{itemize}
\item \textbf{First rigorous error and convergence analysis}: We present the first formal error and convergence guarantees for the NDDV method, modeling data valuation as a discrete-time stochastic optimal control problem. By leveraging the Hamiltonian structure of the problem, we provide a comprehensive analysis.
\item \textbf{Quantitative stability and error bounds}: We establish precise stability and error estimates for loss differences under standard Lipschitz and smoothness conditions. Our results include quadratic bounds with respect to control perturbations and provide insights into the impact of time discretization on accuracy.
\item \textbf{Asymptotic stationarity of training dynamics}: We prove that, under appropriate step-size conditions, the expected squared gradient norm of the training loss converges to zero asymptotically. This ensures the asymptotic stationarity of the NDDV training dynamics.
\item \textbf{Sublinear convergence guarantee}: We demonstrate that the meta-level optimization associated with an $L$-smooth meta-loss achieves a sublinear convergence rate of $O(1/\sqrt{K})$, leading to $\varepsilon$-stationarity in $O(1/\varepsilon^2)$ iterations.
\end{itemize}

\section{Problem Formulation and Preliminaries}

In this section, we introduce the mathematical framework underlying neural dynamic data valuation (NDDV). 
The method reformulates data valuation as a discrete-time stochastic optimal control problem.
whose analysis relies on the associated Hamiltonian system.

\subsection{Controlled Dynamics}

Let $N$ denote the dataset size and let $T\in\mathbb{N}$ be the number of time steps with
step size $\Delta t>0$. For each data point $i=1,\dots,N$, its state $X_{i,t}\in\mathbb{R}^d$
evolves according to the controlled stochastic recursion
\begin{equation}
\label{eq:dynamics}
    X_{i,t+1}
    = X_{i,t} + b\bigl(X_{i,t},\mu_t,\psi_{i,t}\bigr)\,\Delta t 
      + \sigma\,\Delta W_{t},
\end{equation}
where $\psi_{i,t}$ is the control at time $t$, $\sigma$ is a constant diffusion coefficient,
and $\Delta W_t$ are independent Gaussian increments.  
The mean state at time $t$ is denoted by
$
    \mu_t := \frac1N \sum_{i=1}^N X_{i,t}.
$
Throughout, the drift takes the form
\begin{equation}
    b(X_{i,t},\mu_t,\psi_{i,t}) = a(\mu_t - X_{i,t}) + \psi_{i,t},\quad a>0,
\end{equation}
where $a>0$ is a fixed parameter.

\subsection{Cost Functional}

Each trajectory is evaluated through a running cost $R_i(X,\mu,\psi)$ and a terminal 
cost $\Phi_i(X,\mu,\psi)$.  
The discrete objective functional is
\begin{align}
\label{eq:objective}
    \mathcal{L}(\psi,\theta)
    &:= \frac1N \sum_{i=1}^N \Bigg[
        \sum_{t=0}^{T-1} R_i(X_{i,t},\mu_t,\psi_{i,t}) \nonumber\\
        &\quad+ \mathcal{V}\!\left(\Phi_i(X_{i,T},\mu_T,\psi_{i,T});\theta\right) \Phi_i(X_{i,T},\mu_T,\psi_{i,T})
    \Bigg],
\end{align}
where $\mathcal{V}(\cdot;\theta)$ is a weighting function parametrized by~$\theta$.

\subsection{Hamiltonian and Adjoint Variables}

Following the stochastic maximum principle~\cite{kushner1965}, we introduce the Hamiltonian
\begin{equation}
\label{eq:Hamiltonian}
    \mathcal{H}(X_t,Y_t,Z_t,\mu_t,\psi_t)
    = \bigl[a(\mu_t-X_t)+\psi_t\bigr]\cdot Y_t 
    + \sigma^\top Z_t 
    - R(X_t,\mu_t,\psi_t),
\end{equation}
with adjoint variables $(Y_{i,t},Z_{i,t})$ satisfying the backward recursion
\begin{align*}
Y_{i,T} &= -\frac{1}{N}\nabla_X\!\Big(\mathcal{V}(\Phi_i;\theta)\,\Phi_i\Big),  \\
Y_{i,t}&= \nabla_X b(X_{i,t},\mu_t,\psi_{i,t})^\top Y_{i,t+1}
+ \frac{1}{N}\nabla_X R(X_{i,t},\mu_t,\psi_{i,t}),
\end{align*}
for $t=T-1,\dots,0$.

\subsection{Dynamic Valuation Definition}

The dynamic value of point $(x_i,y_i)$ is given by
\begin{equation}
\label{eq:dynamic_value}
    \phi(x_i,y_i;U_i)
    := U_i - \frac{1}{N-1} \sum_{j\ne i} U_j, \quad U_i := -X_{i,T}\cdot Y_{i,T}.
\end{equation}
which enables all valuations to be computed within a single training process.

Eq~\ref{eq:dynamic_value} provides a discrete-time counterpart of the underlying continuous dynamics, yielding a unified framework that enables error estimation and convergence analysis for data valuation within a single training process. 
Using the Half Moons dataset as an example, NDDV is effective for standard data-valuation tasks—detecting corrupted data, removing high- and low-value data, and adding high- and low-value data (see Figure \ref{fig:moon_value})—while substantially improving computational efficiency.

\begin{figure}[h]
    \centering
    \includegraphics[width=0.45\textwidth]{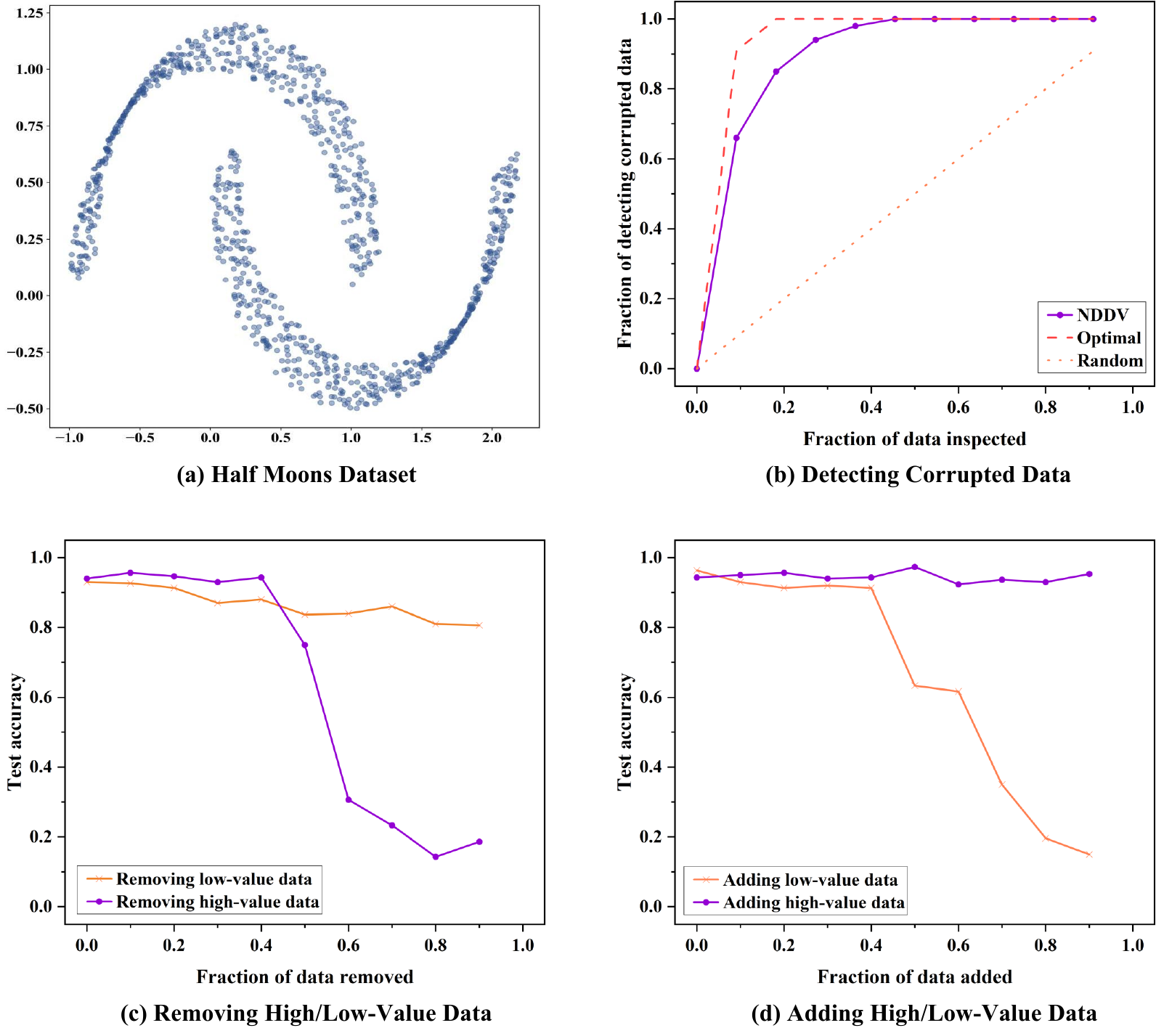}
    \caption{\textbf{The validity of NDDV for data valuation.} On the Half Moons dataset, NDDV effectively evaluates data value in experiments on corrupted-data detection and data removal/addition.}
    \label{fig:moon_value}
\end{figure}

\subsection{Regularity Assumptions}

The analysis relies on the following standard assumptions.

\begin{assumption}
\label{assump:terminal}
The terminal cost $\Phi$ and its gradient are Lipschitz continuous, and
$\Phi$ is twice continuously differentiable.  
Moreover,
\[
    \|\mathcal{V}(\cdot)\bigl(\Phi(X)-\Phi(X')\bigr)\|
    + \|\mathcal{V}(\cdot)\bigl(\nabla\Phi(X)-\nabla\Phi(X')\bigr)\|
    \le K \|X-X'\|.
\]
for some $K>0$.
\end{assumption}

\begin{assumption}
\label{assump:regularity}
The functions $R$ and $b$ are continuous in $t$, twice differentiable and 
Lipschitz in $X$, uniformly in $t$ and $\psi$, and
\begin{align*}
\|b(X)-b(X')\| &+ \|\nabla b(X)-\nabla b(X')\|
+ |R(X)-R(X')| \\
&+ \|\nabla R(X)-\nabla R(X')\|\le K \|X-X'\|.
\end{align*}
The diffusion $\sigma$ is constant.
\end{assumption}

These assumptions ensure well-posedness of the forward–backward system and enable the 
derivation of the error and convergence estimates presented in the subsequent sections.

\section{Error Estimate}

In this section we establish a quantitative stability estimate for the objective functional
$\mathcal{L}$ with respect to perturbations of the control. Throughout, Assumptions
\ref{assump:terminal}–\ref{assump:regularity} are in force.

Let $\psi$ and $\psi'$ denote two admissible control sequences.  
We aim to bound the loss difference $\mathcal{L}(\psi)-\mathcal{L}(\psi')$ in terms of the
discrepancy between these controls.  
The analysis relies on the forward recursion \eqref{eq:dynamics}, the adjoint recursion, and
properties of the Hamiltonian \eqref{eq:Hamiltonian}.

\subsection{Auxiliary Lemma: Discrete Grönwall Inequality}

We begin with a standard discrete Grönwall estimate.

\begin{lemma}[Discrete Grönwall]
\label{lem:gronwall_discrete}
Let $(u_t)_{t=0}^T$ and $(w_t)_{t=0}^{T-1}$ be nonnegative sequences satisfying
\[
    u_{t+1} \le K u_t + w_t, \qquad t=0,\dots,T-1,
\]
for some $K\ge0$. Then
\begin{equation}
    u_t \le \max\{1,K^T\}\Bigl(u_0 + \sum_{s=0}^{T-1} w_s\Bigr),
    \qquad \forall\, 0\le t\le T.
\end{equation}
\end{lemma}

\begin{proof}
Iterating the recursion gives,
\[
u_t \;\le\; K^{t}u_0+\sum_{s=0}^{t-1}K^{\,t-1-s}w_s,\qquad t=0,\dots,T.
\]
Since $K^{t}\le \max\{1,K^{T}\}$ and $K^{\,t-1-s}\le \max\{1,K^{T}\}$, we obtain
\begin{align*}
u_t \;&\le\; \max\{1,K^{T}\}\,u_0+\max\{1,K^{T}\}\sum_{s=0}^{t-1}w_s \\
\;&\le\; \max\{1,K^{T}\}\!\left(u_0+\sum_{s=0}^{T-1}w_s\right).
\end{align*}
which proves the claim.
\end{proof}

\subsection{Bound on the Adjoint Variable}

We first control the magnitude of the co-state variables.

\begin{lemma}
\label{lem:Y_bound}
There exists a constant $C>0$, independent of $i$, $t$, and $N$, such that
\begin{equation}
    \|Y_{i,t}\| \le \frac{C}{N},
    \qquad \forall\, i=1,\dots,N,\quad t=0,\dots,T.
\end{equation}
\end{lemma}

\begin{proof}
From the terminal condition and Assumption~\ref{assump:terminal},  
\[
    \|Y_{i,T}\| = \frac{1}{N}
        \bigl\|\nabla\!\bigl(\mathcal{V}(\Phi_i;\theta)\Phi_i\bigr)\bigr\|
        \le \frac{K}{N}.
\]
For $t<T$, the adjoint recursion gives
\[
    Y_{i,t}
    = \nabla_X b(X_{i,t},\mu_t,\psi'_{i,t})^\top Y_{i,t+1}
      + \frac{1}{N} \nabla_X R(X_{i,t},\mu_t,\psi_{i,t}),
\]
and Assumption~\ref{assump:regularity} implies
\[
    \|Y_{i,t}\| \le K\|Y_{i,t+1}\| + \frac{K}{N},
\]
Applying the discrete Grönwall inequality yields
\begin{align*}
u_t \le \max(1,K^T)\!\left(u_T+\sum_{s=t}^{T-1}w_s\right)\le \frac{K\max(1,K^T)(1+T)}{N} = \frac{C}{N}.
\end{align*}
where $C$ is a constant independent of $i,t,N$.  
Hence $\|Y_{i,t}\|\le C/N$ for all $i,t$.
\end{proof}

\subsection{Main Stability Estimate}

We now compare the objective under two controls.

\begin{lemma}
\label{lem:error_estimate}
There exists a constant $C>0$ such that for all admissible $\psi,\psi'$,
\begin{equation}
    |\mathcal{L}(\psi)-\mathcal{L}(\psi')|
    \le \frac{C}{N}
        \sum_{i=1}^N \sum_{t=0}^{T-1}
        \|\psi_{i,t}-\psi'_{i,t}\|^2.
\end{equation}
\end{lemma}

\begin{proof}
Let $\Delta X_{i,t}=X_{i,t}-X'_{i,t}$, 
$\Delta Y_{i,t}=Y_{i,t}-Y'_{i,t}$, 
$\Delta Z_{i,t}=Z_{i,t}-Z'_{i,t}$.  
Using the decomposition of the loss difference through the Hamiltonian identity yields
\begin{align*}
    \frac{1}{N}\sum_{i,t}
    \bigl[R_i(\cdot,\psi_{i,t})-R_i(\cdot,\psi'_{i,t})\bigr]
    &= -\sum_{i,t}\bigl(\mathcal{H}_i(\psi_{i,t})-\mathcal{H}_i(\psi'_{i,t})\bigr)\\
    &+ \sum_{i,t}\!\Bigl(
        \nabla_X\mathcal{H}_i\cdot\Delta X_{i,t}+ \nabla_Y\mathcal{H}_i\cdot\Delta Y_{i,t+1}
    \Bigr) \\
      &+ \mathcal{R}_T,
\end{align*}
where $\mathcal{R}_T$ is a terminal error term.
By Lipschitz continuity (Assumptions~\ref{assump:terminal}–\ref{assump:regularity}),
\[
    |\mathcal{R}_T|
    \le \frac{C}{N} \sum_{i}\|\Delta X_{i,T}\|^2.
\]
Similarly,
\[
    \Bigl|\sum_{i,t}\sigma^\top \Delta Z_{i,t}\Bigr|
    \le \varepsilon\sum_{i,t}\|\Delta Z_{i,t}\|^2 + C_\varepsilon T\|\sigma\|^2,
\]
and
\[
    \Bigl|\sum_{i,t}\nabla_X\mathcal{H}_i\cdot\Delta X_{i,t}\Bigr|
    + \Bigl|\sum_{i,t}\nabla_Y\mathcal{H}_i\cdot\Delta Y_{i,t+1}\Bigr|
    \le C \sum_{i,t}\bigl(\|\Delta X_{i,t}\|^2+\|\psi_{i,t}-\psi'_{i,t}\|^2\bigr).
\]

The dynamics and adjoint recursions imply
\[
    \|\Delta X_{i,t}\|
    + \|\Delta Y_{i,t}\|
    \le
    C\sum_{s=0}^{T-1}\|\psi_{i,s}-\psi'_{i,s}\|,
\]
Substituting these bounds into the Hamiltonian identity and absorbing constants yields
\[
    \frac{1}{N}\!\sum_{i,t}
    \bigl[R_i(X_{i,t},\mu_t,\psi_{i,t})-R_i(X'_{i,t},\mu'_t,\psi'_{i,t})\bigr]
    \le
    \frac{C}{N}
    \sum_{i,t}\|\psi_{i,t}-\psi'_{i,t}\|^2.
\]
Hence the stated inequality follows.
\end{proof}



Lemma~\ref{lem:error_estimate} shows that the NDDV loss is quadratically stable with respect to control perturbations, which is crucial for the convergence analysis in the next section.

\section{Convergence Analysis}

In this section, we establish the convergence of the NDDV optimization scheme, showing asymptotic vanishing of the training-loss gradient and sublinear $\mathcal{O}(1/\sqrt{K})$ convergence for meta-optimization.


\subsection{Preliminary Lemma}

The proof relies on a standard sequence argument.

\begin{lemma}
\label{lem:seq_convergence}
Let $\{a_n\}_{n\ge1}$ and $\{b_n\}_{n\ge1}$ be nonnegative sequences such that:
(i) $\sum_{n=1}^\infty a_n = \infty$, 
(ii) $\sum_{n=1}^\infty a_n b_n < \infty$, and 
(iii) $|b_{n+1}-b_n|\le K a_n$ for some $K>0$.  
Then $b_n\to0$ as $n\to\infty$.
\end{lemma}

\subsection{Convergence of the Training Loss}

We first consider the update of the control variables.  
Let $\mathcal{L}(\cdot;\theta)$ be the loss defined in \eqref{eq:objective}.  
The control updates take the form
\begin{equation}
    \psi^{k+1}
    = \psi^k - \alpha^k\bigl(\nabla_\psi\mathcal{L}(\psi^k;\theta^{k+1})
     + \upsilon^k\bigr),
\end{equation}
where $\alpha^k$ is the learning rate and $\upsilon^k$ is a zero-mean noise term.

\begin{theorem}
\label{thm:train_convergence}
Suppose $\mathcal{L}$ is $L$-smooth in $\psi$, 
$\|\nabla_\psi\mathcal{L}(\psi;\theta)\|\le\rho$ for all $(\psi,\theta)$,
and the weighting function $\mathcal{V}$ is differentiable with $\rho$-bounded gradient
and bounded Hessian.  
Let the step sizes satisfy
\[
    \alpha^k = \min\!\left\{1,\,\frac{k}{T}\right\}, 
    \qquad
    \beta^k = \min\!\left\{\frac{1}{L},\, \frac{c}{\sigma\sqrt{T}}\right\},
\]
with $c>0$ chosen so that $\sigma\sqrt{T}/c \ge L$.  
Then
\begin{equation}
    \lim_{k\to\infty}
    \mathbb{E}\!\left[\|\nabla\mathcal{L}(\psi^k;\theta^{k+1})\|^2\right] = 0.
\end{equation}
\end{theorem}

\begin{proof}
The Hamiltonian satisfies $\nabla_\psi\mathcal{H}(\cdot,\psi;\theta)
= -\nabla_\psi\mathcal{L}(\psi;\theta)$.  
The $L$-smoothness of $\mathcal{L}$ and the stochastic update give
\begin{align*}
    \mathbb{E}\!\Big[
        \mathcal{H}(\cdot,\psi^{k+1};\theta^{k+1})
        \mid\mathcal{F}_k
    \Big]
    &\le \mathcal{H}(\cdot,\psi^k;\theta^{k+1})\\
    &-\Bigl(\alpha^k - \tfrac{L}{2}(\alpha^k)^2\Bigr)
        \|\nabla_\psi\mathcal{H}(\cdot,\psi^k;\theta^{k+1})\|^2\\
    &+ \tfrac{L}{2}(\alpha^k)^2\sigma^2.
\end{align*}
A similar one-step bound holds for variations in $\theta$ using $\beta^k$ and the
$\delta$-smoothness of $\mathcal{V}$.  
Summing these inequalities from $k=0$ to $K-1$ yields
\[
    \sum_{k=0}^{K-1}
        \Bigl(\alpha^k - \tfrac{L}{2}(\alpha^k)^2\Bigr)
        \mathbb{E}\|\nabla_\psi\mathcal{H}(\cdot,\psi^k;\theta^{k+1})\|^2
    \le C,
\]
with $C$ independent of $K$.  
Since $\alpha^k-\tfrac{L}{2}(\alpha^k)^2 \ge \tfrac12 \alpha^k$,  
\[
    \sum_{k=0}^\infty
        \alpha^k\,
        \mathbb{E}\|\nabla_\psi\mathcal{H}(\cdot,\psi^k;\theta^{k+1})\|^2
        < \infty.
\]

Lipschitz continuity of $\nabla\mathcal{H}$ ensures that  
$|\mathbb{E}[\|\nabla\mathcal{H}^{k+1}\|^2]
    - \mathbb{E}[\|\nabla\mathcal{H}^k\|^2]|
    \le K\alpha^k$.
Applying Lemma~\ref{lem:seq_convergence} shows that  
$\mathbb{E}\|\nabla\mathcal{H}(\cdot,\psi^k;\theta^{k+1})\|^2\to0$,  
and hence
\[
    \mathbb{E}\|\nabla\mathcal{L}(\psi^k;\theta^{k+1})\|^2\to 0.
\]
This completes the proof.
\end{proof}


\subsection{Convergence of the Meta Loss}

We next analyze the update of the meta-parameters~$\theta$:
\begin{equation}
    \theta^{k+1}
    = \theta^k - \beta^k\bigl(\nabla_\theta\ell(\theta^k)+\xi^k\bigr),
\end{equation}
where $\ell$ denotes the meta loss and $\xi^k$ is a zero-mean noise term.

\begin{lemma}
\label{lem:Lipschitz_meta}
Let $\ell$ be $L$-smooth and let $\mathcal{V}$ satisfy
$\|\nabla\mathcal{V}\|\le\delta$ and $\|\nabla^2\mathcal{V}\|\le B$.  
If $\|\nabla_{\hat{\psi}}\ell\|\le\rho$ for all metadata points, then 
$\nabla_\theta \ell$ is Lipschitz with constant
\begin{equation}
    L_V = \alpha\rho^2\bigl(\alpha L\delta^2 + B\bigr).
\end{equation}
\end{lemma}

\begin{proof}
Using the $\theta$-update together with $L$-smoothness of $\ell$ gives
\begin{align*}
    \mathbb{E}\!\left[
        \ell(\hat{\psi}^k(\theta^{k+1})) -
        \ell(\hat{\psi}^k(\theta^k))
    \right]
    &\le
    -\Bigl(\beta^k - \tfrac{L}{2}(\beta^k)^2\Bigr)
    \mathbb{E}\|\nabla\ell(\hat{\psi}^k(\theta^k))\|^2\\
    &\qquad+ \tfrac{L\sigma^2}{2}(\beta^k)^2,
\end{align*}
A second estimate accounts for the variation in $\hat{\psi}$, using
$\|\nabla_\psi\ell\|\le\rho$ and the $\psi$-update.  
Combining the two inequalities yields
\[
    \sum_{k=1}^K
        \Bigl(\beta^k - \tfrac{L}{2}(\beta^k)^2\Bigr)
        \mathbb{E}\|\nabla\ell(\hat{\psi}^k(\theta^k))\|^2
    \le C.
\]
for a constant $C$ depending on the initial loss and noise bounds.
Dividing by $\sum_{k=1}^K(\beta^k - \frac{L}{2}(\beta^k)^2)$ proves the Lipschitz claim.
\end{proof}

\begin{figure*}[h]
    \centering
    \includegraphics[width=0.90\textwidth]{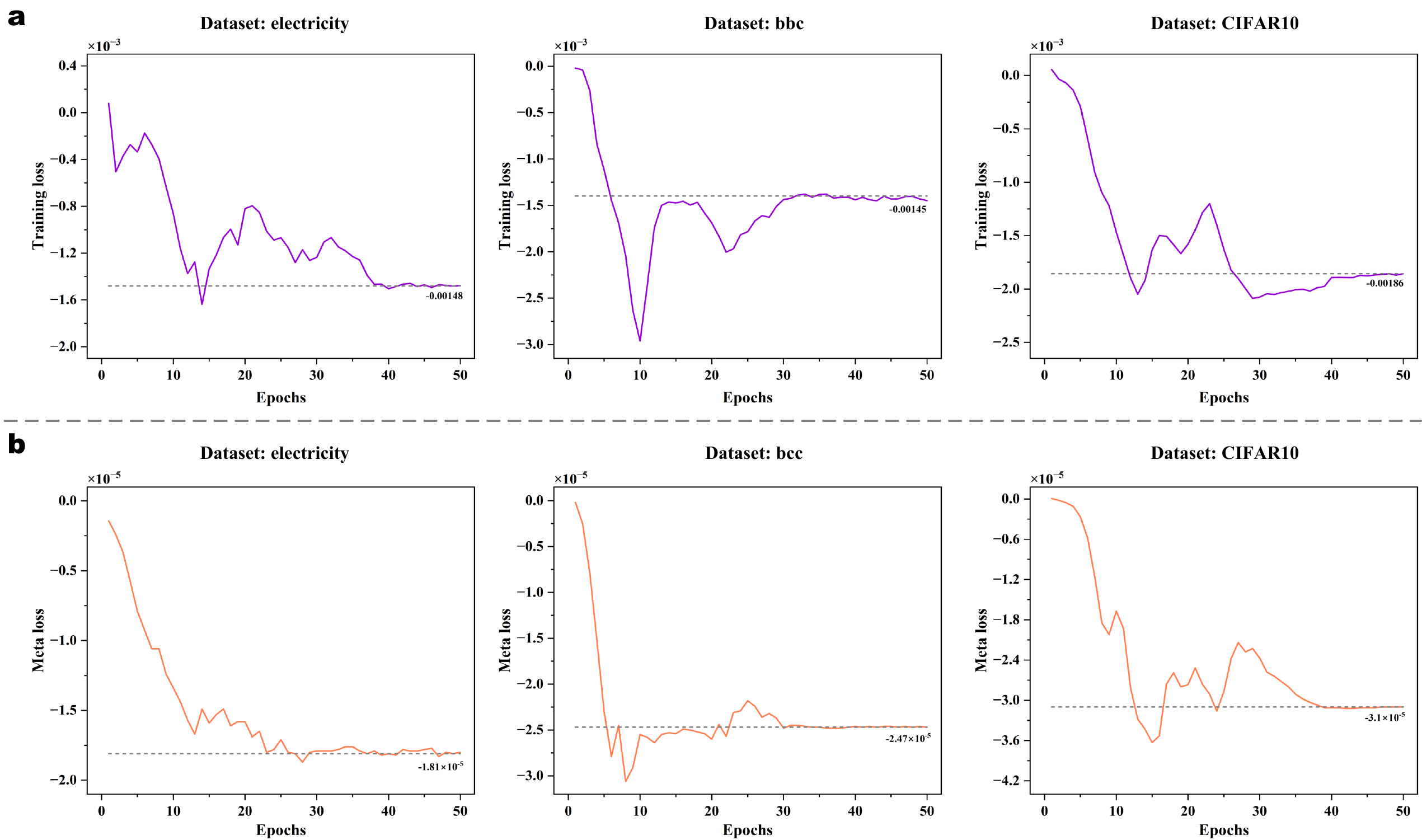}
    \caption{\textbf{Convergence behavior of NDDV.} 
    \textbf{a.} Training loss curves.
    \textbf{b.} Meta loss curves.  
    Both curves illustrate the monotone decay and sublinear convergence predicted by
    the theory.}
    \label{fig:train_meta_loss}
\end{figure*}

\begin{theorem}
\label{thm:meta_convergence}
Let $\ell$ be an $L$-smooth meta loss.  
Under the assumptions of Lemma~\ref{lem:Lipschitz_meta} and the step size choice
\[
    \beta^k = \min\!\left\{\frac1L,\, \frac{c}{\sigma\sqrt{K}}\right\},
\]
there holds
\begin{equation}
    \min_{0\le k\le K}
        \mathbb{E}\|\nabla\ell(\hat{\psi}^k(\theta^k))\|^2
    = \mathcal{O}\!\left(\frac1{\sqrt{K}}\right).
\end{equation}
Thus, $\epsilon$-stationarity is achieved in $\mathcal{O}(1/\epsilon^2)$ iterations.
\end{theorem}

\begin{proof}
The update of $\theta$ is
\begin{equation*}
    \theta^{k+1} = \theta^{k} - \beta^k \big[\nabla \ell(\hat{\psi}^{k}(\theta^{k})) + \xi^{k}\big],
\end{equation*}
where $\xi^k$ is unbiased with $\mathbb{E}[\xi^k]=0$ and $\mathbb{E}\|\xi^k\|^2 \le \sigma^2$.
By $L$-smoothness of $\ell$,
\begin{align*}
\mathbb{E}[\ell(\hat{\psi}^{k}(\theta^{k+1})) - \ell(\hat{\psi}^{k}(\theta^{k}))] 
&\le -(\beta^k - \tfrac{L(\beta^k)^2}{2}) \mathbb{E}\|\nabla \ell(\hat{\psi}^{k}(\theta^{k}))\|^2 \\
&\qquad+ \tfrac{L\sigma^2}{2}(\beta^k)^2.
\end{align*}

For the inner update, smoothness in $\psi$ and $\|\nabla_\psi \ell\|\le\rho$ give
\begin{equation*}
\ell(\hat{\psi}^{k+1}(\theta^{k+1})) - \ell(\hat{\psi}^{k}(\theta^{k+1})) 
\le \alpha^k\rho^2\Big(1+\tfrac{L\alpha^k}{2}\Big),
\end{equation*}

Thus,
\begin{align*}
\sum_{k=1}^K(\beta^k-\tfrac{L(\beta^k)^2}{2})\mathbb{E}\|\nabla \ell(\hat{\psi}^{k}(\theta^{k}))\|^2
&\le \ell(\hat{\psi}^{1}(\theta^{1}))\\
&+ \sum_{k=1}^K\alpha^k\rho^2(1+\tfrac{L\alpha^k}{2}) \\
&+ \tfrac{L\sigma^2}{2}\sum_{k=1}^K(\beta^k)^2,
\end{align*}
Dividing by $\sum_{k=1}^K(\beta^k - \tfrac{L(\beta^k)^2}{2})$ yields
\begin{align*}
\min_{k} \mathbb{E}\|\nabla \ell(\hat{\psi}^{k}(\theta^{k}))\|^2
\le \frac{2\ell(\hat{\psi}^{1}(\theta^{1}))}{K\beta^k} 
+ \frac{2\alpha^1\rho^2(2+L)}{\beta^k} 
+ L\sigma^2\beta^k,
\label{eq:final_bound}
\end{align*}
Choosing $\beta^k=\min\{\frac{1}{L},\frac{c}{\sigma\sqrt{K}}\}$ gives
\begin{align*}
\min_{0\le k\le K}\mathbb{E}\|\nabla \ell(\hat{\psi}^{k}(\theta^{k}))\|^2
&\le \frac{\sigma\,\ell(\hat{\psi}^{1}(\theta^{1}))}{c\sqrt{K}}
+ \frac{K\sigma\rho^2(2+L)}{c\sqrt{K}} 
+ \frac{L\sigma c}{\sqrt{K}} \\
&= \mathcal{O}\!\left(\frac{1}{\sqrt{K}}\right).
\end{align*}
Hence the method attains $\epsilon$-stationarity in $\mathcal{O}(1/\epsilon^2)$ iterations.
\end{proof}

Theorem~\ref{thm:train_convergence} and Theorem~\ref{thm:meta_convergence} together establish asymptotic convergence of the NDDV training dynamics and a sublinear convergence rate for meta-optimization.

\section{Convergence Verification}

In this section, we present numerical experiments illustrating the convergence behavior
predicted by Theorems~\ref{thm:train_convergence} and~\ref{thm:meta_convergence}.  
The experiments are designed to verify that both the training loss and the meta loss exhibit the asymptotic decrease and sublinear convergence rate established in Section~4.

\begin{figure*}[h]
    \centering
    \includegraphics[width=0.80\textwidth]{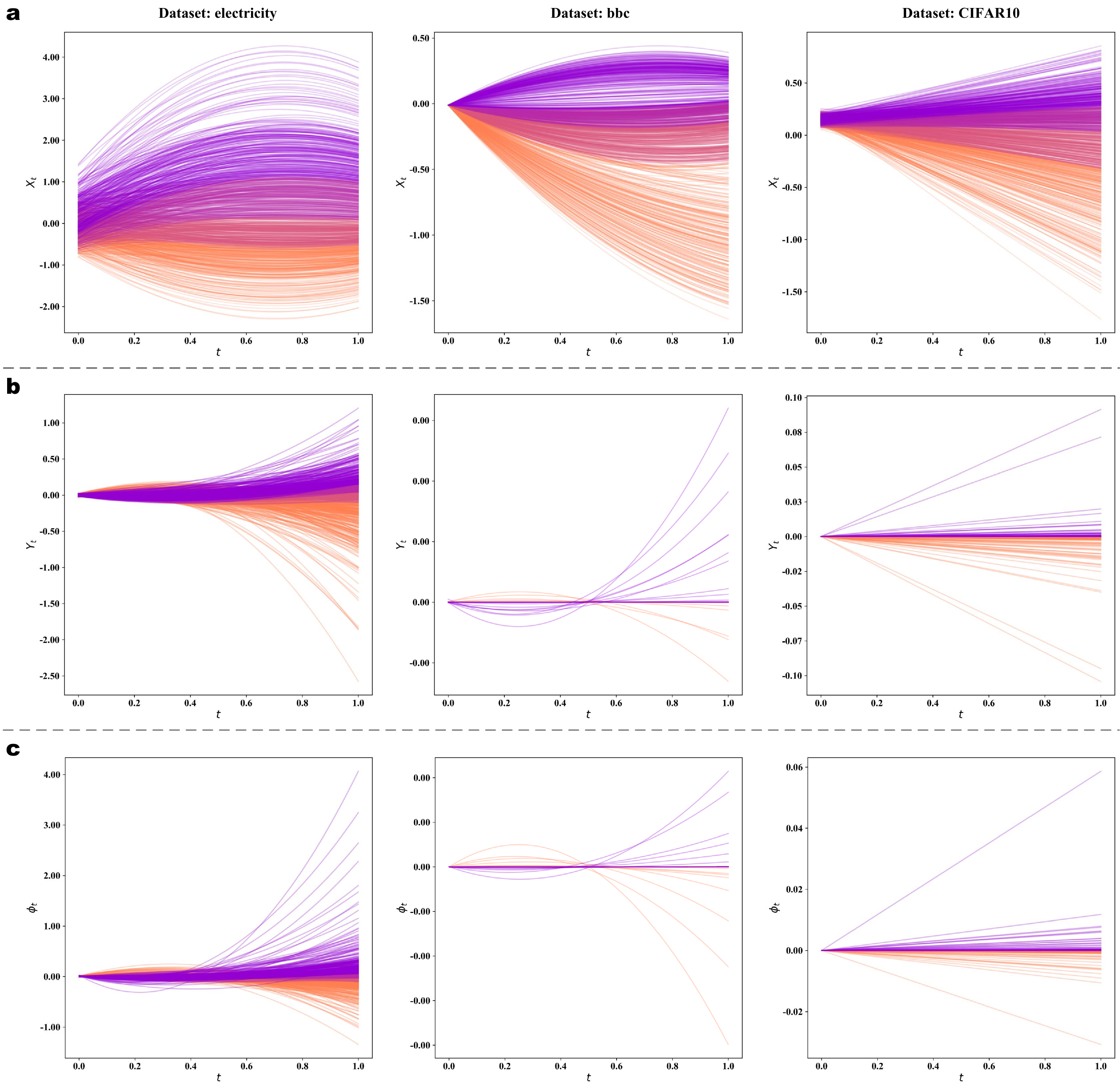}
    \caption{\textbf{State behavior of NDDV.} \textbf{a.} State trajectories.
    \textbf{b.} Co-state trajectories.
    \textbf{c.} Value trajectories.
    Both trajectories illustrate the evolution of each data point’s state and value under the SMP.}
    \label{fig:co_state_control}
\end{figure*}

\subsection{Experimental Setup}
In this section, training is conducted for $50$ epochs, with the meta network updated every $5$ epochs.  
Both the control variables $\psi$ and the meta-parameters $\theta$ are optimized using the Adam method, with initial learning rates $0.01$ (control) and $0.001$ (meta), respectively.
The diffusion coefficient and drift parameter are chosen to be consistent with the dynamics
\eqref{eq:dynamics}.
To ensure stable updates, the data-point weights generated by the valuation network
$\mathcal{V}(\cdot;\theta)$ are normalized at each iteration so that
$\Bigl|\mathcal{V}(\ell;\theta^k)\Bigr| = 1.
$

The normalized weights are computed as
\[
    \eta_i^{k}
    =
    \frac{
        \mathcal{V}(\Phi_i(\cdot,\psi_i);\theta^k)
    }{
        \sum_{j=1}^N \mathcal{V}(\Phi_j(\cdot,\psi_j);\theta^k)
        \;+\;
        \delta\!\left(\sum_{j=1}^N \mathcal{V}(\Phi_j(\cdot,\psi_j);\theta^k)\right)
    }.
\]
where $\delta(\cdot)=\tau>0$ when the denominator is zero and $0$ otherwise.

We illustrate the iterative update process of the state control process in Algorithm \ref{alg:state_process}, comprising the state equation, the costate equation, and the maximization of the Hamiltonian.

\begin{algorithm}[h]
    \label{alg:state_process}
    \caption{State control process}
    \begin{algorithmic}[1]  \small
        \State \textbf{Initialize} $\psi^0=\{\psi^0_t\in\Psi:t=0\dots,T-1\}$.
        \For{$k = 0$ {\bfseries to} $K$}
        \State Solve the forward SDE:
        \begin{align*}
            \mathrm{d}X^k_t=\left[a(\mu_t - X^k_t) + \psi^k_t\right] \mathrm{d}t+\sigma\mathrm{d}W_{t}, \quad X^k_0 = x,
        \end{align*} 
        \State Solve the backward SDE: 
        \begin{align*}
            \mathrm{d}Y^k_t &= -\nabla_x\mathcal{H}\bigl(t,X^k_t,Y^k_t,\mu^k_t,\psi^k_t)\mathrm{d}t + Z^k_t\mathrm{d}W_t,\quad\\
            Y^k_T &= -\mathcal{V}({\Phi(\vcenter{\hbox{.}});\theta})\nabla_x\Phi(X^k_T,\mu^k_T,\psi^k_T),
        \end{align*}
        \State For each $t\in[0,T-1]$, update the state control $\psi^{k+1}_t$:
        \begin{align*}
            \psi^{k+1}_t = \text{arg}\max_{\psi\in\Psi}\mathcal{H} (t,X^k_t,Y^k_t,Z^k_t, \mu_t,\psi).
        \end{align*}
        \normalsize
        \EndFor 
    \end{algorithmic}
\end{algorithm}

Experiments are conducted on three representative datasets (see Table \ref{tab:datasets}), and all results are reported in terms of empirical training-loss and meta-loss trajectories.

\begin{table}
    \centering
    \caption{Selected classification datasets for convergence verification.}
    \resizebox{0.48\textwidth}{!}{
    \label{tab:datasets}
    \begin{tabular}{c|cccccccccc}
        \toprule
        \multirow{2}{*}{\shortstack{Dataset}}
        & Sample & Input & Number of & Minor Class & Data \\
        & Size & Dimension & Classes & Proportion & Type \\
        \midrule
        electricity \cite{gama2004learning} & 38474 & 6 & 2 & 0.5 & Tabular \\
        bbc \cite{greene2006practical} & 2225 & 768 & 5 & 0.17 & Text \\
        CIFAR10 \cite{krizhevsky2009learning} & 50000 & 2048 & 10 & 0.1 & Image \\
        \bottomrule
    \end{tabular}
    }
\end{table}

\subsection{Results}

Figure~\ref{fig:train_meta_loss} displays the convergence curves for both the training
loss and the meta loss.  
Across all datasets tested, the training loss decreases monotonically and stabilizes
as predicted by Theorem~\ref{thm:train_convergence}.  
Similarly, the meta loss exhibits a sublinear convergence profile consistent with the
$\mathcal{O}(1/\sqrt{K})$ rate derived in Theorem~\ref{thm:meta_convergence}.



Figure~\ref{fig:co_state_control}(a) shows that the data-state trajectories evolve from being initially concentrated to becoming increasingly dispersed across all datasets.
This observation suggests that NDDV progressively encodes individual data-point characteristics via the learned data-state trajectories.
Similarly, the data co-state trajectories follow a consistent evolution pattern: they start highly clustered and gradually spread out as training proceeds (Figure~\ref{fig:co_state_control}(b)).
Overall, these results indicate that the data states and co-states evolve jointly, yielding trajectories that are compatible with training convergence. 
Further, according to Eq.~\ref{eq:dynamic_value}, the learnable trajectories enable us to estimate the values of all data points within a single training process.
As shown in Figure~\ref{fig:co_state_control}(c), data value trajectories maintain the trends observed in both data state and co-state trajectories. These trajectories, initially nearly converging at a single point, gradually disperse over time. This result indicates that initial data points display nearly identical values. However, through the assessment by NDDV, the learned data value trajectories progressively diverge, effectively leading to a distinct valuation of data points.

These numerical observations provide empirical support for the theoretical analysis and demonstrate that NDDV achieves stable training and meta-optimization across a variety of settings.

\section{Conclusion}

This work provided a mathematical analysis of the neural dynamic data valuation (NDDV)
framework, which formulates data valuation as a discrete-time stochastic optimal control
problem.  
Under standard Lipschitz and smoothness assumptions, we established a quadratic error
estimate demonstrating the stability of the objective with respect to control perturbations.
We further proved that the expected squared gradient of the training loss vanishes
asymptotically and that the meta-level optimization achieves a sublinear
$\mathcal{O}(1/\sqrt{K})$ convergence rate.  
Numerical experiments verified the predicted convergence behavior and illustrated the
practical stability of the method.

These results constitute the first theoretical guarantees for dynamic data valuation
methods and lay a foundation for further analytical developments.  
Future work may include extensions to continuous-time formulations, sharper error bounds,
or convergence rates under additional structural assumptions.


\section{Acknowledgements}
The authors gratefully acknowledge the financial support from the National Natural Science Foundation of China (12572138).
All the authors appreciate the referees for their valuable suggestions, which further improve the quality of the original manuscript.

\bibliographystyle{plain}
\bibliography{references}

@article{liang2024neural,
  title={Neural Dynamic Data Valuation},
  author={Liang, Zhangyong and Gao, Huanhuan and Zhang, Ji},
  journal={arXiv preprint arXiv:2404.19557},
  year={2024}
}

@article{kushner1965,
  title={On the stochastic maximum principle: Fixed time of control},
  author={Kushner, Harold J},
  journal={Journal of Mathematical Analysis and Applications},
  volume={11},
  pages={78--92},
  year={1965},
  publisher={Academic Press}
}

@inproceedings{ghorbani2019,
  title={Data shapley: Equitable valuation of data for machine learning},
  author={Ghorbani, Amirata and Zou, James},
  booktitle={International conference on machine learning},
  pages={2242--2251},
  year={2019},
  organization={PMLR}
}

@article{shapley1953,
  title={A value for n-person games},
  author={Shapley, Lloyd S},
  journal={Contributions to the Theory of Games},
  volume={2},
  number={28},
  pages={307--317},
  year={1953}
}

@article{pei2020,
  title={A survey on data pricing: from economics to data science},
  author={Pei, Jian},
  journal={IEEE Transactions on knowledge and Data Engineering},
  volume={34},
  number={10},
  pages={4586--4608},
  year={2020},
  publisher={IEEE}
}

@inproceedings{wangdata,
  title={Data Shapley in One Training Run},
  author={Wang, Jiachen T and Mittal, Prateek and Song, Dawn and Jia, Ruoxi},
  booktitle={The Thirteenth International Conference on Learning Representations},
  year={2025}
}

@article{jia2019knn,
  title={Efficient task-specific data valuation for nearest neighbor algorithms},
  author={Jia, Ruoxi and Dao, David and Wang, Boxin and Hubis, Frances Ann and Gurel, Nezihe Merve and Li, Bo and Zhang, Ce and Spanos, Costas J and Song, Dawn},
  journal={arXiv preprint arXiv:1908.08619},
  year={2019}
}

@article{kwon2021beta,
  title={Beta shapley: a unified and noise-reduced data valuation framework for machine learning},
  author={Kwon, Yongchan and Zou, James},
  journal={arXiv preprint arXiv:2110.14049},
  year={2021}
}

@incollection{elkind2016game,
  title={Cooperative game theory},
  author={Elkind, Edith and Rothe, J{\"o}rg},
  booktitle={Economics and computation: an introduction to algorithmic game theory, computational social choice, and fair division},
  pages={135--193},
  year={2016},
  publisher={Springer}
}

@article{peng1990smp,
  title={A general stochastic maximum principle for optimal control problems},
  author={Peng, Shige},
  journal={SIAM Journal on control and optimization},
  volume={28},
  number={4},
  pages={966--979},
  year={1990},
  publisher={SIAM}
}

@inproceedings{gama2004learning,
  title={Learning with drift detection},
  author={Gama, Joao and Medas, Pedro and Castillo, Gladys and Rodrigues, Pedro},
  booktitle={Advances in Artificial Intelligence--SBIA 2004: 17th Brazilian Symposium on Artificial Intelligence, Sao Luis, Maranhao, Brazil, September 29-Ocotber 1, 2004. Proceedings 17},
  pages={286--295},
  year={2004},
  organization={Springer}
}

@inproceedings{greene2006practical,
  title={Practical solutions to the problem of diagonal dominance in kernel document clustering},
  author={Greene, Derek and Cunningham, P{\'a}draig},
  booktitle={Proceedings of the 23rd international conference on Machine learning},
  pages={377--384},
  year={2006}
}

@inproceedings{krizhevsky2009learning,
  title={Learning multiple layers of features from tiny images},
  author={Krizhevsky, Alex and others},
  year={2009},
  publisher={Citeseer}
}

@article{ilyas2022datamodels,
  title={Datamodels: Predicting predictions from training data},
  author={Ilyas, Andrew and Park, Sung Min and Engstrom, Logan and Leclerc, Guillaume and Madry, Aleksander},
  journal={arXiv preprint arXiv:2202.00622},
  year={2022}
}

@article{li2023robust,
  title={Robust data valuation with weighted banzhaf values},
  author={Li, Weida and Yu, Yaoliang},
  journal={Advances in Neural Information Processing Systems},
  volume={36},
  pages={60349--60383},
  year={2023}
}

\end{document}